\documentclass[journal]{IEEEtran}
\usepackage{amsmath,amsfonts,amssymb}
\usepackage{nomencl}
\usepackage{amsthm}
\usepackage{array}
\usepackage{slashed}
\usepackage{tikz}
\usepackage[linesnumbered,ruled,vlined]{algorithm2e}
\usepackage[caption=false,font=normalsize,labelfont=sf,textfont=sf]{subfig}
\usepackage{textcomp}
\usepackage[numbers,sort&compress]{natbib}
\usepackage{stfloats}
\usepackage{enumerate}
\usepackage{booktabs}
\usepackage{tabularx}
\usepackage{array}
\usepackage{makecell}
\usepackage{multirow}
\usepackage{enumitem}
\usepackage{url}
\usepackage{soul}
\usepackage{color, xcolor}
\usepackage{verbatim}
\theoremstyle{definition}
\usepackage{graphicx}
\newtheorem{definition}{Definition}
\newtheorem{example}{Example}

\newtheorem{remark}{Remark}
\newtheorem{theorem}{Theorem}
\newcommand{\tabincell}[2]{\begin{tabular}{@{}#1@{}}#2\end{tabular}}
\newcommand{\circledsmall}[1]{\hbox{\tikz\draw (0pt, 0pt)
    circle (.45em) node {\makebox[0.15em][c]{\scriptsize#1}};}}
\newcommand{\circledtiny}[1]{\hbox{\tikz\draw (0pt, 0pt)
    circle (.35em) node {\makebox[0.15em][c]{\tiny#1}};}}
    \newcommand{\circledlarge}[1]{\hbox{\tikz\draw (0pt, 0pt)
    circle (.6em) node {\makebox[0.15em][c]{\normalsize#1}};}}

\ifCLASSINFOpdf

\else

\fi

\begin{document}

\title{Evidential Information Fusion on Possibilistic Structure}

\author{Anonymous}

\author{Qianli Zhou,
        Ye Cui,
        Zhen Li,
        Witold Pedrycz \IEEEmembership{Life~Fellow,~IEEE},
        Yong Deng
\thanks{Qianli Zhou is with School of Electronics and Information, Northwestern Polytechnical University, Xi'an 710072, China, also with the Department of Electrical and Computer Engineering, University of Alberta, Edmonton, AB T6G 1H9, Canada. }
\thanks{Ye Cui is with the Department of Electrical and Computer Engineering, University of Alberta, Edmonton, AB T6G 1H9, Canada.}
\thanks{Zhen Li is with the China Mobile Information Technology Center, Beijing, 100029, China.}
\thanks{Witold Pedrycz is with the Department of Electrical and Computer Engineering, University of Alberta, Edmonton, AB T6G 2R3, Canada, also with the Systems Research Institute, Polish Academy of Sciences, 00-901 Warsaw, Poland, and Istinye University, Faculty of Engineering and Natural Sciences, Department of Computer Engineering, Sariyer/Istanbul, Turkiye (e-mail: wpedrycz@ualberta.ca).}
\thanks{Yong Deng is with the Institute of Fundamental and Frontier Science, University of Electronic Science and Technology of China, Chengdu 611731, China. (e-mail: dengentropy@uestc.edu.cn)}
}
\markboth{Journal of \LaTeX\ Class Files,~Vol.~14, No.~8, August~2015}%
{Shell \MakeLowercase{\textit{et al.}}: Bare Demo of IEEEtran.cls for IEEE Journals}

\maketitle

\begin{abstract}
Dempster’s rule is a fundamental tool for combining belief functions from distinct and reliable sources. However, its intersection-based semantics imposes strong structural restrictions, which limits its flexibility in handling complex source states and diverse information fusion scenarios. To overcome this limitation, we propose a reversible transformation, derived from the isopignistic principle, between belief functions and a possibilistic structure defined on the power set. In this transformation, the relationships among subsets are explicitly characterized by a belief evolution network, which provides a more flexible representation of evidential information beyond the conventional mass function structure. On this basis, we further introduce the triangular norm family to develop a general and adaptive evidential information fusion framework. Unlike fusion methods rooted in Dempster semantics, the proposed framework supports more flexible combination behaviors and exhibits advantages in non-distinct source fusion, conflict management, parametric combination design, and heterogeneous information fusion.
\end{abstract}


\IEEEpeerreviewmaketitle

\section{Introduction}
\label{intro}

\IEEEPARstart{I}{nformation} fusion plays a central role in uncertain information processing. Among the existing approaches, Dempster's rule \cite{yager2008classic} is one of the most representative tools in Dempster-Shafer theory (DST) for combining belief functions from distinct and reliable sources. Because of its rigorous foundation and consistency with probabilistic information, it has been extensively used in multi-source information fusion \cite{yang2025multimodal,LI2026novel,liang2025assessment}, decision making \cite{zhang2023bsc,yang2025maximum,gao2025information}, social network analysis \cite{wen2024eriue,zhao2024mase}, and risk assessment \cite{huang2024merging,lian2025linguistic,huang2025dynamic}. For instance, Huang \textit{et al.} \cite{huang2024integration,huang2026incomplete} proposed a compelling evidential framework for incomplete-data classification, where multiple imputation methods are integrated with Dempster's rule to mitigate distribution shift and improve classification performance. However, when the source assumptions required by Dempster's rule are violated, particularly under severe evidence conflict, its direct application may produce counter-intuitive fusion results \cite{kang2024deceptive}. Although many variants and improved combination rules have been developed for conflict management \cite{zhou2024generalized,deng2025cross,dong2026quantum}, most of them still operate within the semantic framework induced by Dempster's rule. As a result, their ability to accommodate diverse source states, flexible combination behaviors, and heterogeneous information structures remains limited.

Exploring the combination rules under different states of sources is a significant issue. When the sources are distinct and at least one is reliable, Smets (see \cite{yager2008classic}, Chapter 25) proposes the disjunctive combination rule (DCR). The exploration of combination rules for sources with more general states can be divided into two aspects. On the one hand, from the viewpoint of matrix calculus, the CCR and DCR can be regarded as the specialization and generalization of mass functions \cite{smets2002application}. This concept is further extended to the $\alpha$-junction in evidence combination rules to accommodate states quantified from the perspective of truthfulness \cite{pichon2012alpha}. On the other hand, the diffidence function obtained from the canonical decomposition of a mass function \cite{dubois2020prejudice} is used as the medium to propose the cautious combination rule (CauCR) and bold combination rule (BCR) \cite{denoeux2008conjunctive}. CauCR and BCR aim to resolve the combination of bodies of evidence from non-distinct sources. However, only non-dogmatic mass functions can generate diffidence functions, which prevents these extensions from being directly applied to probability distributions. \textbf{Therefore, the first motivation of this paper is to establish an information fusion framework for non-distinct sources that is applicable to the entire belief function domain.}
 
Possibility theory (PossT), an effective tool for representing incomplete information based on membership functions \cite{li2024basic,li2025z}, is often discussed alongside belief functions. Its contour function can be viewed as a possibility distribution \cite{destercke2011idempotent}, which reflects the propensity of the mass function \cite{zhou2022modeling}. From the perspective of the consonant mass function, when the belief structure is nested, it can be reversibly transformed to a possibility distribution, which is the most widely accepted view of the relationship between DST and PossT \cite{dubois2001new}. In \cite{dubois2008definition}, Dubois \textit{et al.} have demonstrated that the consonant mass function represents the least committed case within its corresponding isopignistic domain. Here, the isopignistic domain refers to the set of belief functions with equal outcomes in pignistic probability transformation (PPT). Thus, in addition to the contour function, there is another correspondence between the possibility distribution and the belief function, where the possibility distribution is induced by the consonant mass function in the isopignistic domain of the belief function. Based on the above correspondence, Smets \cite{smets2000theorie} proposes the hyper-cautious transferable belief model, and the minimum t-norm in possibilistic information fusion is introduced into the belief function from the perspective of the least commitment principle. However, this rule is only applicable to the consonant mass function. In \cite{denoeux2008conjunctive}, executing minimum t-norm on the commonality function has been tried on general mass functions, but the fused outcome cannot be guaranteed to be transformed to a viable mass function. Hence, an alternative method is the diffidence mass function \cite{dubois2020prejudice}. Due to the presence of doubt beliefs, the diffidence function may exceed $1$, which prevents fusion operators designed for possibility distributions from being extended to belief functions directly. \textbf{Therefore, the second motivation of this paper is to provide a pathway for combining evidential information on possibilistic structure to realize flexible combination rules via the triangular norm family.}

In this paper, we first employ the belief evolution network \cite{zhou2022belief} to characterize the structural relationships among subsets within the power set information distribution. Building on this foundation, we propose a reversible evidential information representation based on the isopignistic transformation \cite{zhou2024isopignistic}, which permits any subset of the power set to take values in $[0,1]$ without structural constraints, while ensuring correspondence to a unique valid belief function. Furthermore, we incorporate the triangular norm family as a flexible and versatile possibilistic fusion operator within the proposed representation, thereby realizing evidential information fusion on a possibilistic structure. Compared with Dempster’s semantics, the proposed framework exhibits superior performance in handling non-distinct source fusion, conflict management, parametric combination rules, and heterogeneous information fusion.

The structure of the paper is as follows. Section \ref{pre} introduces the fundamental concepts of combination rules. Section \ref{iso} presents a novel representation of belief functions using isopignistic canonical decomposition. In Section \ref{hcrms}, this representation is extended into an evidential information fusion framework. Section \ref{prop} discusses the fundamental properties and advantages of the proposed combination rules. Section \ref{sec:parametric-pecr} further generalizes the combination rules using parametric t-norms. Section \ref{con} concludes the paper by summarizing the contributions and limitations of this work.

\section{Preliminaries}
\label{pre}

\subsection{Basic probability assignment}
Consider an uncertain variable $X$ that takes values in a finite frame of discernment (FoD) $\Omega=\{\omega_1,\cdots,\omega_n\}$. A mass function of $X$, called a basic probability assignment (BPA), is a mapping $m:2^\Omega\rightarrow[0,1]$ satisfying $\sum_{F_i\subseteq\Omega}m(F_i)=1$ \cite{yager2008classic}. Each $F_i$ is a subset of $\Omega$, where $i\in\{0,\cdots,2^n-1\}$ denotes its binary index. A subset $F_i$ with $m(F_i)>0$ is called a focal set. Several special cases are often considered: when $m(\emptyset)>0$, $m$ is an unnormalized mass function; when $m(\Omega)=1$, $m$ is the vacuous mass function, representing total ignorance; when $m(\emptyset)=1$, $m$ is the empty mass function; when all focal sets are nested, $m$ is a consonant mass function, corresponding to a possibility distribution.

\subsection{Equivalent representations}
Since the uncertainty of information distributions on the power set arises not only from values but also from propositions \cite{su2026dependence,zhan2026ternary}, the specific belief about $X$ cannot be directly reflected. Instead, several equivalent representations have been proposed, which can be reversibly transformed into a mass function:
\begin{small}
\begin{equation}\label{belief_function_eq}
			\begin{aligned}
			& Bel(F_i) = \sum_{\emptyset \neq F_j \subseteq F_i} m(F_j), 
		 b(F_i) = Bel(F_i)+m(\emptyset) = \overline{q}(\overline{F_i}), \\
			& Pl(F_i) = \sum_{F_j \cap F_i \neq \emptyset} m(F_j), 
			\sigma(F_i) = \prod_{F_i \subseteq F_j} q(F_j)^{(-1)^{|F_j|-|F_i|-1}}, \\
			& q(F_i) = \sum_{F_i \subseteq F_j} m(F_j), 
			v(F_i) = \prod_{F_j \subseteq F_i} b(F_j)^{(-1)^{|F_i|-|F_j|-1}}.
		\end{aligned}
\end{equation}
\end{small}
where the calculation of $\sigma$ and $v$ is also called canonical decomposition of mass functions \cite{denoeux2008conjunctive}.

\subsection{Evidence combination rules}
Combination rules \cite{pichon2014consistency} are commonly considered depending on the dependency and reliability of sources: the conjunctive combination rule (CCR, $\circledsmall{$\cap$}$) \footnote{Also known as the unnormalized Dempster's rule}, the cautious conjunctive rule (CauCR, $\circledsmall{$\wedge$}$), the disjunctive combination rule (DCR, $\circledsmall{$\cup$}$), and the bold conjunctive rule (BCR, $\circledsmall{$\vee$}$) \cite{denoeux2008conjunctive}. The specific formulas and corresponding states are shown in Table \ref{tab:ecr_states}, where ${F_i}^{\sigma}$ and ${F_i}_{v}$ are simple mass function and its duality, ${F_i}^{\sigma}\equiv\{m(F_i)=1-\sigma,\,m(\Omega)=\sigma\}$ and  ${F_i}_{v}\equiv\{m(F_i)=1-v,\,m(\emptyset)=v\}$.

\begin{table}[htbp!]
	\centering
	\caption{Combination rules under different source states}
	\label{tab:ecr_states}
	\renewcommand{\arraystretch}{1.22}
	\newcommand{\RuleCell}[2]{#1: $#2$}
	\begin{tabularx}{\columnwidth}{
			@{}
			>{\centering\arraybackslash}p{0.10\columnwidth}|
			>{\centering\arraybackslash}X
			>{\centering\arraybackslash}X
			@{}
		}
		\toprule
		State 
		& All reliable 
		& At least one reliable \\
		\midrule
		Distinct
		&
		\RuleCell{CCR}{
			\circledlarge{$\cap$}_{F\subset\Omega}
			F^{\sigma_1(F)\cdot\sigma_2(F)}
		}
		&
		\RuleCell{DCR}{
			\circledlarge{$\cup$}_{F\neq\emptyset}
			F_{v_1(F)\cdot v_2(F)}
		}
		\\
		\midrule
		Non-distinct
		&
		\RuleCell{CauCR}{
			\circledlarge{$\cap$}_{F\subset\Omega}
			F^{\min(\sigma_1(F),\sigma_2(F))}
		}
		&
		\RuleCell{BCR}{
			\circledlarge{$\cup$}_{F\neq\emptyset}
			F_{\min(v_1(F),v_2(F))}
		}
		\\
		\bottomrule
	\end{tabularx}
\end{table}

\section{Representing belief functions on possibilistic structure}
\label{iso}

\subsection{Motivation}

As discussed in Section \ref{intro}, this paper is driven by two motivations: to construct an evidential information fusion framework that can handle non-distinct sources, and to enable this framework to operate on a possibilistic structure through operators from the triangular norm family. To this end, a fundamental requirement is to establish an information-preserving and reversible representation that maps a belief function into a possibilistic form while still allowing the reconstruction of a valid belief function \cite{huang2026individual}.

The propositions in a belief structure are not mutually exclusive, and the assignment and updating of mass values must satisfy structural constraints among subsets. This makes it difficult to directly apply point-wise triangular norms to belief functions, since such operators are designed for $[0,1]$-valued degrees without explicitly considering inter-propositional dependencies. Therefore, this paper does not perform fusion directly on the belief structure. Instead, a BPA is first transformed into a possibilistic structural representation with the same reconstructable information content. Fusion is then performed in this transformed space, and the fused representation is finally reconstructed into a valid BPA.

Following this rationale, we model the dependencies among propositions as a directed acyclic graph, termed the belief evolution network (BEN) \cite{zhou2022belief}. Each directed edge in the BEN corresponds to an inclusion relation between two adjacent cardinality layers and describes how belief mass is transmitted during refinement, as illustrated in Fig. \ref{sd_fig}. Based on this formulation, a reversible bridge between belief structures and possibilistic structures can be constructed.

\begin{figure}[htbp!]
	\centering
	\includegraphics[width=0.35\textwidth]{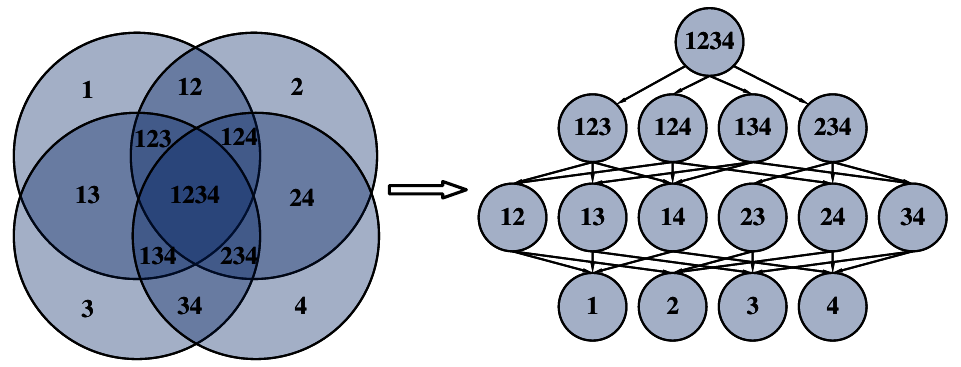}
	\caption{Mapping structural dependency among focal sets into BEN.}
	\label{sd_fig}
\end{figure}

\subsection{Belief evolution network-based isopignistic transformation}
\label{isosubsec}

In this paper, the empty set mass is allowed and denoted by $e_m=m(\emptyset)$, which can be interpreted as the conflict mass or the degree of non-normalizedness. The normalized BPA is defined as $\bar m(F)=m(F)/(1-e_m)$ for $F\neq\emptyset$. The normalized pignistic probability transformation (PPT) is given by
\begin{equation}
	\label{betpen}
	BetP^{\rm N}_m(\omega)
	=
	\sum_{\omega\in F\subseteq\Omega,F\neq\emptyset}
	\frac{\bar m(F)}{|F|}
	=
	\frac{BetP_m(\omega)}{1-m(\emptyset)} .
\end{equation}
Its unnormalized version is defined as
\begin{equation}
	\label{betpe}
	BetP_m(\omega)
	=
	\sum_{\omega\in F\subseteq\Omega,F\neq\emptyset}
	\frac{m(F)}{|F|}.
\end{equation}
Here, $BetP^{\rm N}_m$ is a probability distribution on $\Omega$. When the empty set is explicitly considered, we let $BetP_m(\emptyset)=m(\emptyset)$.

\begin{definition}[Isopignistic domain]
	Given a probability distribution $p$ on an $n$-element FoD $\Omega$ and a fixed empty-set mass $e$, the isopignistic domain associated with $(e,p)$ is defined as
	\[
	\mathfrak{Iso}_{e,p}
	=
	\{m\mid m(\emptyset)=e,\ BetP^{\rm N}_m=p\}.
	\]
\end{definition}
To move a BPA within an isopignistic domain while preserving the same pignistic probability, Zhou \textit{et al}. proposed a BEN-based mechanism, termed the isopignistic transformation \cite{zhou2024isopignistic}.

\begin{definition}[Isopignistic transformation]
	\label{isod}
	Let $m_1,m_2\in\mathfrak{Iso}_{e,p}$ be two mass functions in the same isopignistic domain. Define a trans-isopignistic function $\zeta:2^\Omega\setminus\{\emptyset,\{\omega\}\mid\omega\in\Omega\}\rightarrow\mathbb R$. For any non-empty $F\subseteq\Omega$, the transformation $m_2=T_{\zeta}(m_1)$ is given by
	\[
	m_2(F)=m_1(F)-\zeta(F)+\sum_{G\in{\rm Par}(F)}\frac{\zeta(G)}{|G|},
	\]
	where $\zeta(F)=0$ for $|F|\leq1$, and ${\rm Par}(F)=\{G\subseteq\Omega\mid F\subset G,\ |G|=|F|+1\}$. When $m_1$ and $m_2$ are determined, $\zeta$ can be recursively obtained by
	\[
	\zeta(F)=m_1(F)+\sum_{G\in{\rm Par}(F)}\frac{\zeta(G)}{|G|}-m_2(F).
	\]
\end{definition}

\begin{remark}
	The isopignistic transformation is reversible. For $m_2=T_{\zeta}(m_1)$, the inverse transformation is $m_1=T_{-\zeta}(m_2)$, as proven in \cite{zhou2024isopignistic}.
\end{remark}

\begin{example}
	\label{ex1}
	Consider two mass functions
	\[
	\begin{aligned}
		m_1&=\{0.02,0.10,0.10,0.25,0.06,0.27,0.02,0.18\},\\
		m_2&=\{0.02,0.145,0.02,0.02,0,0,0,0.795\}.
	\end{aligned}
	\]
	For $m_2=T_{\zeta}(m_1)$, the corresponding trans-isopignistic function is
	\[
	\begin{aligned}
		\zeta(\{1,2,3\})&=0.18-0.795=-0.615,\\
		\zeta(\{1,2\})&=0.25-0.615/3-0.02=0.025,\\
		\zeta(\{1,3\})&=0.27-0.615/3-0=0.065,\\
		\zeta(\{2,3\})&=0.02-0.615/3-0=-0.185.
	\end{aligned}
	\]
\end{example}

Definition \ref{isod} and Example \ref{ex1} show that the trans-isopignistic function describes mass redistribution along the BEN while preserving the pignistic probability. However, this construction requires a given target mass function $m_2$, and is therefore insufficient for fusion scenarios where only source BPAs are available. Therefore, a complete semantic representation with a $2^n$-dimensional distribution is required for modeling mass functions in the isopignistic domain.

\subsection{Isopignistic canonical decomposition}

The above analysis shows that $\zeta$ is defined only after both the source and target mass functions are specified. To obtain a structure-oriented representation from a single BPA, an isopignistic canonical decomposition is introduced under the BEN framework. Canonical decomposition represents a belief function by elementary components with explicit semantics \cite{han2019decombination}. Representative frameworks have been developed by Smets \cite{smets1995canonical} and Pichon \cite{pichon2018canonical}; however, they are mainly designed for evidence combination and conflict management, and do not explicitly connect belief, probabilistic, and possibilistic structures \cite{zhou2024isopignistic}.

\begin{definition}[Isopignistic canonical decomposition]
	\label{iso_def}
	Consider a BPA $m$ on an $n$-element FoD $\Omega$, and let its normalized pignistic probability be $p_m$. The corresponding possibility distribution satisfies $Poss_m(\omega_i)=\sum_{\omega_j\in\Omega}\min(p_m(\omega_i),p_m(\omega_j))$. The isopignistic canonical decomposition of $m$ is defined as
	\begin{equation}
		I_m(F)=
		\begin{cases}
			e_m, & F=\emptyset,\\
			Poss_m(\omega), & F=\{\omega\},\\
			\sum_{A\supseteq F}\bar m(A)/\binom{|A|}{|F|}, & |F|\geq2.
		\end{cases}
	\end{equation}
	where $Poss_m$ denotes the consonant possibility distribution induced by $p_m$. The resulting function $I_m$ is called the isopignistic function of $m$.
\end{definition}

The isopignistic function $I_m$ decomposes a BPA into three interpretable components: \textbf{normalizedness}, \textbf{propensity}, and \textbf{commitment}. The \textbf{normalizedness component} $I_m(\emptyset)=e_m$ characterizes the empty-set mass, which reflects conflict or non-normalizedness. The \textbf{propensity component} $I_m(\{\omega\})=Poss_m(\omega)$ describes singleton-level atomic propensity. The \textbf{commitment component} $I_m(F)$ for $|F|\geq2$ characterizes the higher-order commitment carried by composite subset $F$. It is not a belief mass directly assigned to $F$, but a refinement flow through $F$ in the BEN.

The higher-order part of $I_m$ is directly related to the trans-isopignistic function in Definition \ref{isod}. If the target mass function is chosen as the pignistic mass function $m_p$, where $m_p(\{\omega\})=p_m(\omega)$ and $m_p(F)=0$ for $|F|\geq2$, then the trans-isopignistic function transforming $\bar m$ into $m_p$ satisfies $\zeta_m(F)=I_m(F)$ for $|F|\geq2$.

This decomposition provides a canonical representation from a single BPA. However, the components of $I_m$ cannot be adjusted independently without constraints. An arbitrary modification of $I_m$ may violate the non-negativity of the reconstructed BPA, as shown below.

\begin{example}
	\label{iso_counter_example}
	Continuing from Example \ref{ex1}, the isopignistic function of $m_1$ is
	\[
	I_{m_1}=\{0.02,1,0.872,0.316,0.811,0.337,0.082,0.184\}.
	\]
	If $I_{m_1}(\{1,2\})=0.6$ while the other components are kept unchanged, the reconstructed BPA becomes
	\[
	m=\{0.02,-0.039,-0.039,0.528,0.060,0.270,0.020,0.180\}.
	\]
	It can be observed that $m(\{1\})<0$ and $m(\{2\})<0$, which violates the non-negativity requirement of a BPA. Therefore, a directly modified isopignistic function is not necessarily admissible.
\end{example}

\subsection{Belief function on possibilistic structure}

Although $I_m$ decomposes a BPA into interpretable components, arbitrary modifications of $I_m$ may violate BPA reconstructability. Therefore, to obtain a reconstructable representation on the possibilistic structure, we further transform $I_m$ into an isopignistic relative function $\widehat I_m$, whose layer-wise structural pattern is constrained by the BEN capacity. For ease of definition, the following notations are introduced.

\begin{itemize}[leftmargin=1pt]
	\item \textbf{$t$-th cardinality layer $\mathcal L^{(t)}$.}
	$\mathcal L^{(t)}=\{F\subseteq\Omega:|F|=t\}$ is the collection of focal sets with cardinality $t$. Each $\mathcal L^{(t)}$ is regarded as a layer-specific finite frame.
	
	\item \textbf{$t$-th layer BEN transmit amount $m^{(t)}_{BetP}$.}
	It denotes the information available at the $t$-th layer for further transmission. For $t=1$, $m^{(1)}_{BetP}(\{\omega\})=p_m(\omega)=BetP_m^{\rm N}(\omega)$. For $t\geq2$, $m^{(t)}_{BetP}(F)=I_m(F)$, $F\in\mathcal L^{(t)}$.
	
	\item \textbf{$t$-th order isopignistic function $I_m^{(t)}$.}
	It is the restriction of $I_m$ to $\mathcal L^{(t)}$, i.e., $I_m^{(t)}=\{I_m(F):F\in\mathcal L^{(t)}\}$.
	
	\item \textbf{Parent set $\mathrm{Par}(F)$.}
	$\mathrm{Par}(F)=\{G\subseteq\Omega:G=F\cup\{\omega\},\omega\in\Omega\setminus F\}$.
\end{itemize}

\begin{definition}[Isopignistic relative function]
	\label{iso_rel_def}
	Consider a mass function $m$ on an $n$-element FoD $\Omega$ and its isopignistic function $I_m$. The isopignistic relative function $\widehat I_m$ is constructed as follows.
	
	\textbf{Step 1:} Preserve the normalizedness and propensity components:
$
	\widehat I_m(\emptyset)=I_m(\emptyset),\quad
	\widehat I_m(\{\omega\})=I_m(\{\omega\}),\ \forall\omega\in\Omega
$.
	
	\textbf{Step 2:} For each $t=1,\ldots,n-1$, define the bottleneck focal set as
	\[
	F^\star=\arg\min_{F\in\mathcal L^{(t)}}
	\frac{m^{(t)}_{BetP}(F)}{\sum_{G\in{\rm Par}(F)}I_m(G)}.
	\]
	When the denominator is zero, the corresponding channel is regarded as inactive and excluded from the bottleneck search; if all parent sums are zero, the $(t+1)$-th layer is set to zero.
	
	\textbf{Step 3:} Define the scaling coefficient as
	\[
			s_m(t)=
		\frac{\sum_{G\in{\rm Par}(F^\star)}I_m(G)}
		{(t+1)m^{(t)}_{BetP}(F^\star)}.
\]
	Here, $0\leq s_m(t)\leq1$, which measures the admissible activation degree of the $(t+1)$-th commitment layer.
	
	\textbf{Step 4:} For each $G\in\mathcal L^{(t+1)}$, define
	\[
	\widehat I_m(G)=
	\frac{s_m(t)I_m(G)}{\max_{H\in\mathcal L^{(t+1)}}I_m(H)}.
	\]
	This operation preserves the relative commitment pattern within the same cardinality layer.
\end{definition}

\begin{example}[Isopignistic relative function of $m_1$]
	\label{iso_rel_example}
	Continuing from Examples \ref{ex1} and \ref{iso_counter_example}, the isopignistic function of $m_1$ is
	\[
	I_{m_1}=\{0.02,1,0.872,0.316,0.811,0.337,0.082,0.184\}.
	\]
	The normalizedness and propensity components are preserved. For $t=1$, the singleton-layer transmit amounts are
	\[
	m^{(1)}_{BetP}=\{0,0.429,0.301,0,0.270,0,0,0\}.
	\]
	Since
	\[
	\frac{m^{(1)}_{BetP}(\{3\})}{I_{m_1}(\{1,3\})+I_{m_1}(\{2,3\})}
	=
	\min_{F\in\mathcal L^{(1)}}
	\frac{m^{(1)}_{BetP}(F)}{\sum_{G\in{\rm Par}(F)}I_{m_1}(G)},
	\]
	the bottleneck focal set is $F^\star=\{3\}$ and
	\[
	s_{m_1}(1)=
	\frac{I_{m_1}(\{1,3\})+I_{m_1}(\{2,3\})}
	{2m^{(1)}_{BetP}(\{3\})}
	=0.776.
	\]
	Thus,
	\[
	\widehat I_{m_1}(\{1,2\})=0.728,
	\widehat I_{m_1}(\{1,3\})=0.776,
	\widehat I_{m_1}(\{2,3\})=0.189.
	\]
	For $t=2$, $F^\star=\{2,3\}$ and
	\[
	s_{m_1}(2)=
	\frac{I_{m_1}(\{1,2,3\})}{3I_{m_1}(\{2,3\})}
	=
	0.748.
	\]
	Finally,
	\[
	\widehat I_{m_1}
	=
	\{0.02,1,0.872,0.728,0.811,0.776,0.189,0.748\}.
	\]
\end{example}

Compared with $I_m$, the relative representation $\widehat I_m$ preserves the normalizedness and propensity components exactly. The commitment component is rescaled, but its relative magnitudes within each cardinality layer remain unchanged. Hence, $\widehat I_m$ only separates the structural commitment pattern from the admissible scale determined by the BEN capacity.

\begin{definition}[Reconstruction of mass function]
	\label{iso_rel_inv_def}
	Let $\widehat I:2^\Omega\rightarrow[0,1]$ be an isopignistic relative function. The reconstructed BPA is denoted by $m=\mathcal R(\widehat I)$.
	
	\textbf{Step 1:} Let $\pi(\omega)=\widehat I(\{\omega\})$ and sort the elements as $\pi(\omega_{(1)})\geq\cdots\geq\pi(\omega_{(n)})$. Recover the pignistic probability by
	\[
	p(\omega_{(n)})=\frac{\pi(\omega_{(n)})}{n},
	p(\omega_{(r)})=p(\omega_{(r+1)})
	+\frac{\pi(\omega_{(r)})-\pi(\omega_{(r+1)})}{r}.
	\]
	
	\textbf{Step 2:} Assign $I(\emptyset)=\widehat I(\emptyset)$ and $I(\{\omega\})=\widehat I(\{\omega\})$.
	
	\textbf{Step 3:} For each $t=1,\ldots,n-1$, define
	\[
	F^\star=\arg\min_{F\in\mathcal L^{(t)}}
	\frac{m^{(t)}_{BetP}(F)}
	{\sum_{G\in{\rm Par}(F)}\widehat I(G)},
	r(t)=
	\frac{(t+1)m^{(t)}_{BetP}(F^\star)}
	{\sum_{G\in{\rm Par}(F^\star)}\widehat I(G)}.
	\]
	As in Definition~\ref{iso_rel_def}, channels with zero parent sum are treated as inactive; if all parent sums are zero, the corresponding recovered commitment layer is set to zero.
	Then, for each $G\in\mathcal L^{(t+1)}$,
	\[
	I(G)=r(t)\widehat I(G)
	\max_{H\in\mathcal L^{(t+1)}}\widehat I(H).
	\]
	
	\textbf{Step 4:} Reconstruct the normalized BPA from $I$ by
	\[
	\bar m(\Omega)=I(\Omega),
	\bar m(F)=I(F)-\sum_{G\in{\rm Par}(F)}\frac{I(G)}{|G|}.
	\]
	
	\textbf{Step 5:} Let $e=\widehat I(\emptyset)$. The final BPA is
	\[
	m(\emptyset)=e,\qquad m(F)=(1-e)\bar m(F),\quad F\neq\emptyset.
	\]
\end{definition}

\begin{theorem}
	\label{thm:valid_reconstruction}
	Let $\widehat I:2^\Omega\rightarrow[0,1]$ satisfy $\widehat I(\emptyset)\in[0,1]$ and $\max_{\omega\in\Omega}\widehat I(\{\omega\})=1$. The reconstruction result $m=\mathcal R(\widehat I)$ is a valid BPA, i.e., $m(F)\geq0$ for all $F\subseteq\Omega$ and $\sum_{F\subseteq\Omega}m(F)=1$.
\end{theorem}

\begin{proof}
	If $\widehat I(\emptyset)=1$, the reconstruction gives $m(\emptyset)=1$ and $m(F)=0$ for $F\neq\emptyset$, which is valid. Otherwise, $\widehat I(\emptyset)\in[0,1)$. Since $\max_{\omega\in\Omega}\widehat I(\{\omega\})=1$, the singleton layer is a normalized possibility distribution and can be converted into a valid probability distribution $p$. For each higher-order layer, the bottleneck-based reverse scaling ensures
	$
	\sum_{G\in{\rm Par}(F)}I(G)
	\leq
	(t+1)m^{(t)}_{BetP}(F),\quad \forall F\in\mathcal L^{(t)}$.
	Thus, the recovered commitment flow never exceeds the BEN capacity, and the reconstructed conditional BPA satisfies $\bar m(F)\geq0$ and $\sum_{F\neq\emptyset}\bar m(F)=1$. Finally, $m(\emptyset)=\widehat I(\emptyset)$ and $m(F)=(1-\widehat I(\emptyset))\bar m(F)$ yield a valid BPA.
\end{proof}

Although the isopignistic relative function $\widehat I_m$ maps each subset to a value in $[0,1]$, it is still defined on the original belief structure. The elements of $2^\Omega$ are not mutually orthogonal propositions. To make it compatible with possibilistic processing, each cardinality layer $\mathcal L^{(t)}$ is regarded as an independent finite frame, whose elements are not original propositions to be assigned belief masses, but structural channels with the same cardinality. On each such layer, the restriction of $\widehat I_m$ can be represented as a layer-wise possibilistic representation.

\begin{definition}[Layer-wise possibilistic representation]
	\label{def:layer_poss_rep}
	Let $\Omega$ be a finite FoD, a layer-wise possibilistic representation on $\mathcal L^{(t)}$ is a mapping $\pi^{(t)}:\mathcal L^{(t)}\rightarrow[0,1]$.
	For any $\mathcal A\subseteq\mathcal L^{(t)}$, the induced layer-wise possibility measure is defined as
	\[
	\Pi^{(t)}(\mathcal A)=\max_{\omega\in\mathcal A}\pi^{(t)}(\omega),
	\]
	with $\Pi^{(t)}(\emptyset)=1-h^{(t)}$, where $h^{(t)}=\Pi^{(t)}(\Omega)$.
	If $h^{(t)}=1$, the representation is normal; otherwise, it is subnormal.
\end{definition}

\begin{definition}[belief possibilistic representation]
	\label{def:belief_poss_rep}
	Let $m$ be a BPA on an $n$-element FoD $\Omega$, and let $\widehat I_m$ be its isopignistic relative function. The $t$-layer possibilistic representation of $m$ is defined as
	\[
	\pi_m^{(t)}(F)=\widehat I_m(F),\quad F\in\mathcal L^{(t)},\quad t=1,\ldots,n.
	\]
	which constitute the belief possibilistic representation
	\[
	\varpi_m
	=
	\left(
	e_m,\pi_m^{(1)},\pi_m^{(2)},\ldots,\pi_m^{(n)}
	\right),
	\]
	where $e_m=m(\emptyset)$.
\end{definition}

The above definition clarifies the relationship among the BPA $m$, the isopignistic relative function $\widehat I_m$, and belief possibilistic representation $\varpi_m$, which have identical information content
\[
m\leftrightarrow I_m\leftrightarrow \widehat I_m\leftrightarrow \varpi_m.
\]

\section{Evidential information fusion on possibilistic structure}\label{hcrms}

\subsection{Fusion framework}

\begin{figure*}[t]
	\centering
	\includegraphics[width=\textwidth]{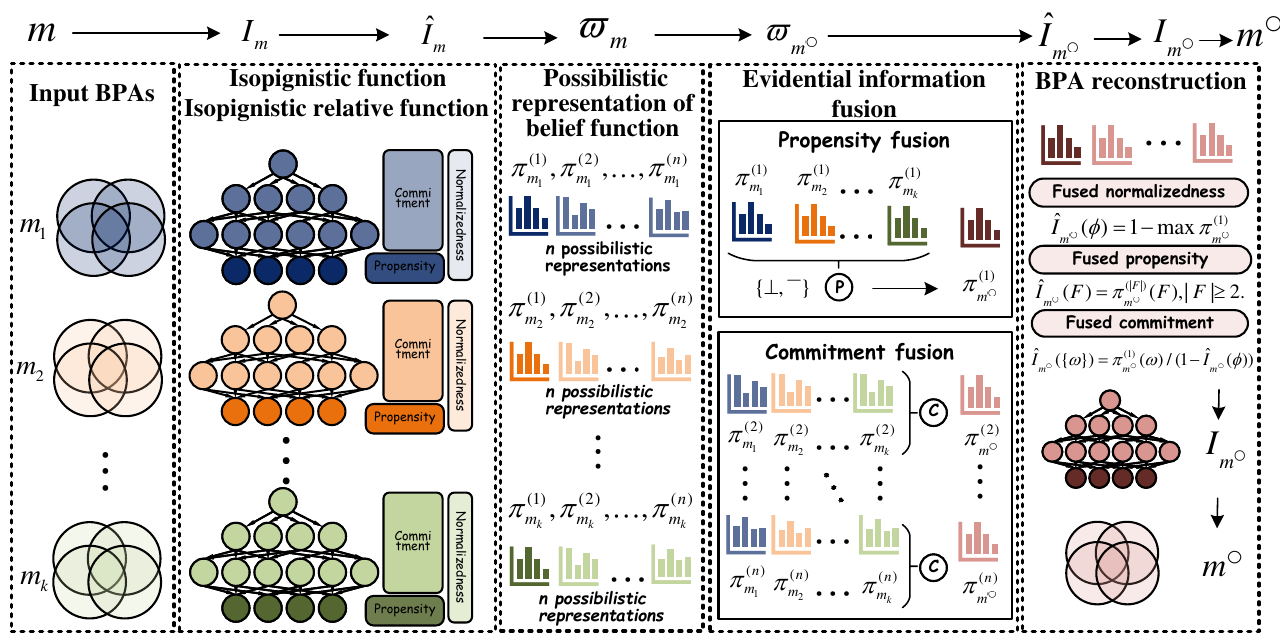}
	\caption{Overview of the proposed method.}
	\label{overview_fig}
\end{figure*}
Based on the proposed possibilistic representations, we further develop a flexible evidential information fusion framework. The overall framework is illustrated in Fig. \ref{overview_fig}. Suppose that $k$ BPAs $m_1,m_2,\ldots,m_k$ are defined on the same $n$-element FoD $\Omega$. For each source $m_i$, its representation process is written as
\[
m_i
\longrightarrow
I_{m_i}
\longrightarrow
\widehat I_{m_i}
\longmapsto
\varpi_{m_i}.
\]
The fusion module consists of two parts: propensity fusion and commitment fusion. The propensity fusion module acts on the singleton layer, which aggregates the effective singleton-level tendency of different sources. The commitment fusion module acts on the higher-layer possibilistic representations. These layers correspond to the commitment component of the isopignistic relative function and describe the relative structural profile generated by the isopignistic refinement process. 

After propensity fusion and commitment fusion, the raw fused singleton propensity profile is denoted by $\widetilde\pi^{(1)\circ}$. Its height determines the empty-set mass:
\[
h^\circ=\max_{\omega\in\Omega}\widetilde\pi^{(1)\circ}(\{\omega\}),\qquad
e_{m^\circ}=1-h^\circ.
\]
When $h^\circ>0$, the normalized singleton profile is
\[
\pi^{(1)\circ}(\{\omega\})
=
\frac{\widetilde\pi^{(1)\circ}(\{\omega\})}{h^\circ}.
\]
If $h^\circ=0$, then $e_{m^\circ}=1$ and the fused result is the empty BPA. The fused possibilistic representation is obtained as
\[
\varpi^\circ=
\left(
e_{m^{\circ}},
\pi^{(1)\circ},\pi^{(2)\circ},\ldots,\pi^{(n)\circ}
\right).
\]
It can be directly rewritten as the fused isopignistic relative function $\widehat I_{m^\circ}$. Specifically,
\[
\widehat I_{m^\circ}(\emptyset)=e_{m^{\circ}},
\qquad
\widehat I_{m^\circ}(\{\omega\})=\pi^{(1)\circ}(\{\omega\}),
\]
and for each higher layer,
\[
\widehat I_{m^\circ}(F)
=
\pi^{(t)\circ}(F),
\quad
F\in\mathcal L^{(t)},\quad t=2,\ldots,n.
\]
Finally, the fused isopignistic relative function is reconstructed into the fused BPA:
\[
\varpi^\circ
\longrightarrow
\widehat I_{m^\circ}
\longrightarrow
I_{m^\circ}
\longrightarrow
m^\circ =m_1\circledtiny{$\pi$}\cdots\circledtiny{$\pi$}m_k.
\]
Therefore, the proposed framework performs fusion in the layer-wise possibilistic representation space while guaranteeing that the final output is reconstructed as a valid belief function.

\subsection{Propensity fusion}

\begin{definition}[Propensity fusion]
	\label{def:propensity_fusion}
	Consider $k$ BPAs defined on the FoD $\Omega$. For source $m_i$, let $\pi_{m_i}^{(1)}$ be its layer-$1$ possibilistic representation. The fused propensity component is obtained pointwise by
	\[
	\widetilde\pi^{(1)\circ}(\{\omega\})
	=
	\circledlarge{\rm P}
	\left(
	\widetilde\pi_{m_1}^{(1)}(\{\omega\}),
	\widetilde\pi_{m_2}^{(1)}(\{\omega\}),
	\ldots,
	\widetilde\pi_{m_k}^{(1)}(\{\omega\})
	\right),
	\] where $\widetilde\pi_{m_i}^{(1)}(\{\omega\})
	=
	(1-e_{m_i})\pi_{m_i}^{(1)}(\{\omega\})$ and $\circledsmall{\rm P}$ is the propensity fusion operator, which can be selected as a t-norm or a t-conorm according to the source state. 
\end{definition}

The propensity fusion module aggregates the singleton-level tendency of different sources. The factor $(1-e_m)$ discounts the singleton possibilistic representation, so only the usable part of the evidence contributes to the fusion of atomic hypotheses. The operator $\circledsmall{\rm P}$ determines the fusion attitude: a t-norm emphasizes singleton hypotheses jointly supported by the sources, whereas a t-conorm preserves singleton hypotheses supported by at least one source.

\subsection{Commitment fusion}

\begin{definition}[Commitment fusion]
	\label{def:commitment_fusion}
	Consider $k$ BPAs defined on the FoD $\Omega$. For source $m_i$, let $\pi_{m_i}^{(t)}$ be its layer-$t$ possibilistic representation, where $t=2,\ldots,n$. The fused commitment component in layer $\mathcal L^{(t)}$ is obtained pointwise by
	\[
	\widetilde\pi^{(t)\circ}(F)
	=
	\circledlarge{\rm C}
	\left(
	\pi_{m_1}^{(t)}(F),
	\pi_{m_2}^{(t)}(F),
	\ldots,
	\pi_{m_k}^{(t)}(F)
	\right),
	\quad F\in\mathcal L^{(t)},
	\]
	where $\circledsmall{\rm C}$ is the commitment fusion operator, which can be selected as a t-norm or a t-conorm according to the source state.
\end{definition}

The commitment fusion module aggregates the higher-order structural commitment of different sources. For each layer $\mathcal L^{(t)}$, $t\geq2$, the value $\pi_{m_i}^{(t)}(F)$ represents the relative commitment strength carried by the structural channel $F$ in the isopignistic refinement process. Therefore, commitment fusion is not a direct fusion of belief masses, but a pointwise aggregation of the layer-wise relative commitment profiles. The operator $\circledsmall{\rm C}$ determines the fusion attitude: a t-norm emphasizes structural channels commonly supported by the sources, whereas a t-conorm preserves structural channels supported by at least one source.

\subsection{Possibilistic evidence combination rule}

\begin{definition}[Possibilistic evidence combination rule]
	\label{def:pecr}
	Consider $k$ BPAs defined on the FoD $\Omega$. For each source $m_i$, its possibilistic representation is  $\varpi_{m_i}=(e_{m_i},\pi_{m_i}^{(1)},\ldots,\pi_{m_i}^{(n)})$. Based on Definitions \ref{def:propensity_fusion} and \ref{def:commitment_fusion}, for layer $t$, the raw fused component is $\widetilde\pi^{(t)\circ}$. Let
	\[
	h^\circ=\max_{\omega\in\Omega}\widetilde\pi^{(1)\circ}(\{\omega\}),\qquad
	e_{m^\circ}=1-h^\circ.
	\]
	For $h^\circ>0$, define $\pi^{(1)\circ}=\widetilde\pi^{(1)\circ}/h^\circ$; if $h^\circ=0$, the fused BPA is the empty BPA. For $t\geq2$, let $\pi^{(t)\circ}=\widetilde\pi^{(t)\circ}$. The possibilistic evidence combination rule (PECR) is defined as
	\[
	\varpi^\circ=
	\left(e_{m^{\circ}},
	\pi^{(1)\circ},
	\pi^{(2)\circ},
	\ldots,
	\pi^{(n)\circ}
	\right),
	\]
	It is then written back as the fused isopignistic relative function:
	\[
	\widehat I_{m^\circ}(\emptyset)=e_{m^\circ},
	\qquad
	\widehat I_{m^\circ}(F)=\pi^{(t)\circ}(F),
	\quad F\in\mathcal L^{(t)}.
	\]
	Finally, the fused BPA is reconstructed by
	\[
	m_1\circledtiny{$\pi$}\cdots\circledtiny{$\pi$}m_k=m^\circ=\mathcal R(\widehat I_{m^\circ}),
	\]
	where $\mathcal R$ is the reconstruction operator defined in Definition~\ref{iso_rel_inv_def}.
\end{definition}

\begin{example}
	\label{ex:pecr}
	Consider two BPAs
	\[
	\begin{aligned}
		m_1&=\{0,0.1,0.12,0.25,0.06,0.27,0.02,0.18\},\\
		m_2&=\{0,0.02,0.16,0.14,0.11,0.31,0.25,0.01\}.
	\end{aligned}
	\]When the propensity operator is product t-norm \footnote{$\circledtiny{\rm P}(x,y)=xy$} and the commitment operator is selected as the maximum t-conorm\footnote{$\circledtiny{\rm C}(x,y)=\max(x,y)$}, their PECR is shown as follows.
	Based on Definition \ref{iso_def}, their isopignistic functions are
	\[\begin{aligned}
		&I_{m_1}
		=
		\{0,1,0.8950,0.3100,0.7950,0.3300,0.0800,0.1800\},\\
		&I_{m_2}
		=
		\{0,0.7450,0.9650,0.1433,1,0.3133,0.2533,0.0100\}.
	\end{aligned}
	\]
	Based on Definition \ref{iso_rel_def}, the corresponding isopignistic relative functions are
	\[\begin{aligned}
		&\widehat I_{m_1}
		=
		\{0,1,0.8950,0.7267,0.7950,0.7736,0.1875,0.7500\},\\
		&\widehat I_{m_2}
		=
		\{0,0.7450,0.9650,0.4206,1,0.9195,0.7434,0.0233\}.
	\end{aligned}
	\]
	Based on Definition \ref{def:belief_poss_rep}, their possibilistic representations are
\vspace{-1.2em}
\[
\begin{minipage}{0.48\linewidth}
	\begin{align*}
		\pi_{m_1}^{(1)}&=(1.000,0.895,0.795),\\
		\pi_{m_1}^{(2)}&=(0.727,0.774,0.188),\\
		\pi_{m_1}^{(3)}&=(0.750), e_m=0.
	\end{align*}
\end{minipage}
\hfill
\begin{minipage}{0.48\linewidth}
	\begin{align*}
		\pi_{m_2}^{(1)}&=(0.745,0.965,1.000),\\
		\pi_{m_2}^{(2)}&=(0.421,0.920,0.743),\\
		\pi_{m_2}^{(3)}&=(0.023), e_m=0.
	\end{align*}
\end{minipage}
\]
Based on the Definitions \ref{def:propensity_fusion} and \ref{def:commitment_fusion}, the fused  components are
	\[\begin{aligned}
			&\widetilde\pi^{(1)\circ}=(0.7450,0.8637,0.7950),\\
		&\widetilde\pi^{(2)\circ}=(0.7267,0.9195,0.7434),\widetilde\pi^{(3)\circ}=(0.7500).
	\end{aligned}
	\]
Based on the Definition \ref{def:pecr}, the fused isopignistic relative function gives
\[
\widehat I_{m^\circ}
=
\{0.136,0.863,1.000,0.727,0.921,0.920,0.743,0.750\}.
\]
Based on the Definition \ref{iso_rel_inv_def}, the fused BPA is
\[
m^\circ
=
\{0.136,0.020,0.138,0.050,0.043,0.104,0.055,0.454\}.
\]
\end{example}
For a more comprehensive comparison, the CCR and DCR are also applied to the same two BPAs in Example \ref{ex:pecr}.
\[
\begin{aligned}
	&m_{\circledtiny{$\cap$}}
	=
	\{0.195,0.177,0.205,0.063,0.166,0.142,0.050,0.002\},\\
	&m_{\circledtiny{$\cup$}}
	=
	\{0,0.002,0.019,0.129,0.007,0.181,0.078,0.584\}.
\end{aligned}
\]
The three methods are compared in Table \ref{tab:method_compare}, where the ignorance degree $\mathrm{Ign}(m)=\sum_{\emptyset\neq F\subseteq\Omega}m(F)|F|$,
and the Shannon entropy is calculated from the normalized pignistic probability $BetP^{\rm N}$.
\begin{table}[htbp]
	\centering
	\caption{Comparison of different fusion methods}
	\label{tab:method_compare}
	\renewcommand{\arraystretch}{1.15}
	\setlength{\tabcolsep}{3pt}
	\small
	\begin{tabular}{c|c|c|c|c}
		\hline
		Method & $m(\emptyset)$ & $\mathrm{Ign}(m)$ & $BetP^{\rm N}$ & $H(BetP^{\rm N})$ \\
		\hline
		Proposed & $0.1363$ & $1.9802$ & $(0.317,0.396,0.288)$ & $1.5716$ \\
		CCR & $0.1951$ & $1.0636$ & $(0.327,0.325,0.348)$ & $1.5842$ \\
		DCR & $0$ & $2.5564$ & $(0.331,0.318,0.352)$ & $1.5837$ \\
		\hline
	\end{tabular}
\end{table}
As shown in Table \ref{tab:method_compare}, the three methods reflect different fusion semantics. CCR produces the largest empty set mass and the smallest ignorance degree, indicating that conjunctive intersection tends to convert disagreement into conflict and concentrate mass on smaller focal sets. In contrast, DCR produces no empty set conflict but gives the largest ignorance degree, since disjunctive union preserves more mass on composite focal sets. The proposed rule shows an intermediate but more flexible behavior. Its product t-norm in the propensity component extracts the common singleton-level tendency, yielding the lowest Shannon entropy of $BetP^{\rm N}$. Meanwhile, its maximum t-conorm in the commitment component preserves structural channels supported by either source, resulting in an ignorance degree between CCR and DCR. Therefore, the proposed rule has more flexible semantics in uncertain information processing, which combines conjunctive reasoning at the propensity level with disjunctive preservation at the commitment level, rather than behaving as a purely conjunctive or disjunctive rule.
\section{Properties and advantages}
\label{prop}
 In this section, we first analyze whether the proposed method satisfies the general fusion properties and then demonstrate its typical advantages.

\subsection{Justification via basic principles}
Based on the criteria outlined in \cite{dubois2016basic} and \cite{abellan2021combination}, we validate the proposed method and assess its ability to combine bodies of evidence in a reasonable manner.

\begin{theorem}[Commutativity]
	The PECR satisfies commutativity, i.e., $m_1\circledsmall{$\pi$}m_2=m_2\circledsmall{$\pi$}m_1$.
\end{theorem}

\begin{proof}
	The propensity component is fused pointwise by $\circledsmall{\rm P}$, and each commitment component is fused pointwise by $\circledsmall{\rm C}$. Since t-norms and t-conorms are commutative, when BPAs are fused in different orders, the fused possibilistic representation is unchanged. Since the reconstruction from the fused isopignistic relative function to the BPA is deterministic, the final BPA is also unchanged.
\end{proof}

\begin{theorem}[Neutrality]
	The PECR has an operator-dependent neutral element. Specifically,
	\[
	\begin{array}{c|c|c}
		\hline
		\circledsmall{\rm P} & \circledsmall{\rm C} & \text{neutral BPA} \\
		\hline
		\text{t-norm} & \text{t-norm} & m(\Omega)=1 \\
		\text{t-norm} & \text{t-conorm} & m(\{\omega\})=1/|\Omega|,\ \forall \omega\in\Omega \\
		\text{t-conorm} & \text{t-conorm} & m(\emptyset)=1 \\
		\text{t-conorm} & \text{t-norm} & \text{/}\\
		\hline
	\end{array}
	\]
\end{theorem}

\begin{proof}
	A neutral BPA must provide the neutral element of the corresponding operator in both the propensity and commitment fusion modules. For the propensity fusion, if operator is a t-norm, the neutral profile must be $1$ for all singletons; if operator is a t-conorm, the neutral profile must be $0$ for all singletons. For the commitment fusion, if operator is a t-norm, the neutral commitment profile must satisfy $\pi_m^{(t)}(F)=1$; if $\circledsmall{\rm C}$ is a t-conorm, it must satisfy $\pi_m^{(t)}(F)=0$.
	
	When both $\circledsmall{\rm P}$ and $\circledsmall{\rm C}$ are t-norms, the vacuous BPA $m(\Omega)=1$ induces an all-one singleton propensity and all-one commitment profiles, and is therefore neutral. When $\circledsmall{\rm P}$ is a t-norm and $\circledsmall{\rm C}$ is a t-conorm, the uniform Bayesian BPA $m(\{\omega\})=1/|\Omega|$ has an all-one singleton propensity and zero higher-order commitment, and is therefore neutral. When both operators are t-conorms, the empty BPA $m(\emptyset)=1$ gives zero effective singleton propensity and zero higher-order commitment, and is therefore neutral. When $\circledsmall{\rm P}$ is a t-conorm and $\circledsmall{\rm C}$ is a t-norm, neutrality would require zero effective singleton propensity but all-one higher-order commitment profiles. The former condition forces the source to be empty, while the latter cannot be induced by the empty BPA. Hence, no universal neutral BPA exists in this case.
\end{proof}

\begin{theorem}[Idempotency]\label{theo:ide}
	If both $\circledsmall{\rm P}$ and $\circledsmall{\rm C}$ are chosen as idempotent operators (min/max), then the PECR satisfies idempotency, i.e., $m\circledsmall{$\pi$}m=m$.
\end{theorem}

\begin{proof}
	If propensity operator is idempotent, then $\circledsmall{\rm P}(x,x)=x$.
	Thus, fusing the singleton profile with itself leaves the raw fused propensity unchanged. If commitment operator is idempotent, then $\circledsmall{\rm C}(y,y)=y$.
	Thus, every fused commitment layer remains identical to the original one. Since the mapping between $m$ and $\widehat I_m$ is one-to-one, the reconstructed BPA is still $m$.
\end{proof}

\begin{theorem}[Associativity]
	When the propensity and commitment operators are fixed, the PECR satisfies associativity, i.e., $(m_1\circledsmall{$\pi$}m_2)\circledsmall{$\pi$}m_3
	=
	m_1\circledsmall{$\pi$}(m_2\circledsmall{$\pi$}m_3)$.
\end{theorem}

\begin{proof}
	Since t-norms and t-conorms are associative, the fused isopignistic relative function is independent of the bracketing order.
	By the one-to-one correspondence between $\widehat I_m$ and $m$, the reconstructed BPAs are identical. Hence, the possibilistic combination rule satisfies associativity.
\end{proof}

\begin{theorem}[Informative monotonicity]
	The PECR satisfies informative monotonicity in the layer-wise possibilistic ordering. Specifically, define $m_1\sqsubseteq_{\varpi}m_2$
	if
	\[\begin{aligned}
		&(1-e_{m_1})\pi_{m_1}^{(1)}(\{\omega\})
		\leq
		(1-e_{m_2})\pi_{m_2}^{(1)}(\{\omega\}),
		\quad \forall \omega\in\Omega,\\
		&\pi_{m_1}^{(t)}(F)
		\leq
		\pi_{m_2}^{(t)}(F),
		\quad \forall F\in\mathcal L^{(t)},\ t\geq2.
	\end{aligned}
	\]
	If both propensity and commitment operators are t-norms, then $m_1\circledsmall{$\pi$}m_2\sqsubseteq_{\varpi}m_i,\quad i=1,2$.
	If both propensity and commitment operators are t-conorms, then $m_i\sqsubseteq_{\varpi}m_1\circledsmall{$\pi$}m_2,\quad i=1,2$.
\end{theorem}

\begin{proof}
A t-norm contracts each layer-wise possibilistic representation, whereas a t-conorm expands it. Therefore, when both operators are t-norms, the fused result is less than or equal to each input under $\sqsubseteq_{\varpi}$. When both operators are t-conorms, each input is less than or equal to the fused result under $\sqsubseteq_{\varpi}$. This proves informative monotonicity.
\end{proof}

\begin{table*}[htbp!]
\caption{Properties' comparison of evidence combination rules in DST}
    \label{tab:my_tcrs}
    \centering
    \begin{tabular}{c|cccccccc}
        \hline
        Rules & $\circledsmall{$\pi$}$   & \circledsmall{$\cap$}/ \circledsmall{$\cup$}  & \circledsmall{$\wedge$}/ \circledsmall{$\vee$}  & $\sqcap_k$/ $\sqcup_k$ & $\oplus_\mathrm{D\&P}$ & $\oplus_\mathrm{Y}$ & $\oplus_\mathrm{Deng}$ & $\oplus_\mathrm{A}$\\
        \hline
        Commutativity & All & Yes & Yes & Yes & Yes & Yes & Yes & Yes \\
        \hline
        Neutrality &  Operator-dependent & Yes & No & Yes & Yes & Yes & No & No \\
        \hline
        Idempotency &  Min/Max & No & Yes & Yes & No & No & No & Yes \\
        \hline 
        Associativity &  Yes & Yes & Yes & Quasi & Quasi & Quasi & Quasi & Quasi\\
        \hline
        \tabincell{c}{Informative monotonicity} &  $\varpi$  & $Q$ / $Bel$ & $\sigma$ / $v$ & None & None & None & None & None\\
        \hline
    \end{tabular}
\end{table*}

Table \ref{tab:my_tcrs} compares the main properties of several representative evidence combination rules in DST, including the proposed possibilistic evidence  combination rule $\circledsmall{$\pi$}$, diffidence function-based rules $(\circledsmall{$\cap$},\circledsmall{$\cup$},\circledsmall{$\wedge$},\circledsmall{$\vee$})$ \cite{denoeux2008conjunctive}, distance-based idempotent rules $(\sqcap_k,\sqcup_k)$ \cite{klein2018idempotent}, Dubois and Prade's rule $\oplus_{\mathrm{D\&P}}$ \cite{dubois2016basic}, Yager's rule $\oplus_{\mathrm{Y}}$ \cite{yager1987dempster}, Deng \textit{et al.}'s rule $\oplus_{\mathrm{Deng}}$ \cite{xiao2024complex}, and Abell{\'a}n \textit{et al.}'s rule $\oplus_{\mathrm{A}}$ \cite{abellan2021combination}. 

As shown in Table \ref{tab:my_tcrs}, the proposed rule satisfies the basic requirements of an evidence combination rule more comprehensively than existing mainstream rules. It satisfies commutativity and associativity, has operator-dependent neutrality, and becomes idempotent when the min/max operators are selected. Moreover, its informative monotonicity can be explicitly characterized by the layer-wise possibilistic representation $\varpi$. In contrast, several existing rules either lack associativity, fail to provide idempotency, or do not offer a clear informative monotonicity criterion. Therefore, the proposed rule provides a more flexible and structurally interpretable combination mechanism while preserving the essential algebraic properties required for evidential information fusion.

\subsection{Advantages}
\label{subsec:further_advantages}

This subsection further discusses three advantages of the proposed PECR.

\subsubsection{Fusion with non-distinct sources on the entire belief function domain}

As discussed in Section \ref{intro}, the first motivation of this paper is to establish an evidential fusion framework for non-distinct sources that is applicable to the entire belief function domain. In contrast to CauCR and BCR in Table \ref{tab:ecr_states}, which require non-dogmatic or unnormalized mass functions, the proposed scheme is not restricted by these conditions. Consider the following Bayesian mass functions:
\[
\begin{aligned}
	m_{1}&=\{0,0.55,0.30,0,0.15,0,0,0\},\\
	m_{2}&=\{0,0.10,0.65,0,0.25,0,0,0\}.
\end{aligned}
\]
Both mass functions assign zero mass to the empty set and the total ignorance set, and thus do not satisfy the usual applicability conditions of CauCR and BCR.
For comparison, the classical Bayesian combination rule, i.e., the normalized conjunctive rule on Bayesian masses, gives
\[
m_{\rm Bayes}
=
\{0,0.1913,0.6783,0,0.1304,0,0,0\}.
\]
Its unnormalized conjunctive version is
\[
m_{\rm Bayes}^{u}
=
\{0.7125,0.0550,0.1950,0,0.0375,0,0,0\}.
\]
Using the proposed framework, when propensity operators are chosen as the minimum t-norm and the maximum t-conorm, respectively, the fused results are
\[
\begin{aligned}
	m_{\min}&=\{0.25,0.10,0.475,0,0.175,0,0,0\},\\
	m_{\max}&=\{0,0.40,0.40,0,0.20,0,0,0\}.
\end{aligned}
\]
The result $m_{\min}$ preserves the same dominant singleton as the Bayesian rule with a more conservative mass allocation. Thus, the proposed rule with minimum t-norm induces a more cautious semantics than Bayesian product rule. Together with the idempotency result in Theorem \ref{theo:ide}, this shows that the proposed framework provides an effective mechanism for fusing non-distinct sources without being restricted to non-dogmatic mass functions.

Another important observation is that the outputs remain probabilistic whenever the inputs are Bayesian. Since $m_1$ and $m_2$ contain no multi-element focal sets, the proposed fusion rule does not introduce additional masses on multi-element propositions. This indicates that the proposed representation and fusion process are compatible with the probability domain and preserve the probabilistic semantic framework when no epistemic uncertainty is provided by the sources.
\subsubsection{Clear and monotone conflict representation}

The second advantage is that the conflict produced by the proposed rule has a clear source. For normalized inputs, $e_{m_i}=m_i(\emptyset)=0$. According to Definition~\ref{def:propensity_fusion}, the fused raw singleton propensity is
\[
\widetilde\pi^{(1)\circ}(\{\omega\})
=
\circledsmall{\rm P}
\left(
\pi_{m_1}^{(1)}(\{\omega\}),\ldots,
\pi_{m_k}^{(1)}(\{\omega\})
\right).
\]
The empty set mass of the fused BPA is then determined by the height of this raw propensity profile:
\[
e_{m^\circ}
=
m^\circ(\emptyset)
=
1-\max_{\omega\in\Omega}
\widetilde\pi^{(1)\circ}(\{\omega\}).
\]

\begin{theorem}[Conflict ordering]
	\label{thm:conflict_ordering}
	For normalized input BPAs, assume that two propensity operators $\circledsmall{\rm P}_a$ and $\circledsmall{\rm P}_b$ satisfy
\[
\circledsmall{\rm P}_a(\mathbf{x})
\le
\circledsmall{\rm P}_b(\mathbf{x}),
\quad \forall \mathbf{x}\in[0,1]^k .
\]
	Let $e_a$ and $e_b$ be the empty set masses generated by the proposed rule under $\circledsmall{\rm P}_a$ and $\circledsmall{\rm P}_b$, respectively, then $e_a\ge e_b$.
\end{theorem}

\begin{proof}
	For every $\omega\in\Omega$, the point-wise order of the operators gives
\[
\circledsmall{\rm P}_a\!\left((\pi_{m_i}^{(1)}(\{\omega\}))_{i=1}^k\right)
\le
\circledsmall{\rm P}_b\!\left((\pi_{m_i}^{(1)}(\{\omega\}))_{i=1}^k\right).
\]
	Taking the maximum over $\Omega$ preserves the inequality. Since the conflict is computed by one minus this maximum height, the order is reversed, which yields $e_a\ge e_b$.
\end{proof}
For comparison, using the two BPAs in Example~\ref{ex:pecr}, Table~\ref{tab:conflict_compare_caucr} lists the results of several rules on the same BPAs. In this table, the rules denoted by $T_{\mathrm{prod}}$, $S_{\mathrm{prob}}$, and $T_{\min}$ are special cases of the proposed PECR. Both the propensity operator and the commitment operator are indicated with the rule name. For example, $T_{\mathrm{prod}}$ means $\circledsmall{\rm P}=\circledsmall{\rm C}=T_{\mathrm{prod}}$.

\begin{table}[htbp]
	\centering
	\caption{Comparison of conflict and uncertainty on Example~\ref{ex:pecr}}
	\label{tab:conflict_compare_caucr}
	\renewcommand{\arraystretch}{1.15}
	\setlength{\tabcolsep}{2.5pt}
	\small
	\begin{tabular}{l|ccc|l|ccc}
		\hline
		Rules & $m(\emptyset)$ & $\mathrm{Ign}(m)$ & $H_{\rm BetP}$
		& Rules & $m(\emptyset)$ & $\mathrm{Ign}(m)$ & $H_{\rm BetP}$ \\
		\hline
		CCR
		& 0.195 & 1.064 & 1.5842
		& $T_{\mathrm{prod}}$
		& 0.136 & 1.268 & 1.5716 \\
		
		CauCR
		& 0.980 & 0.021 & 1.4688
		& $T_{\min}$
		& 0.105 & 1.344 & 1.5620 \\
		
		DCR
		& 0.000 & 2.556 & 1.5836
		& $S_{\mathrm{prob}}$
		& 0.000 & 2.586 & 1.5850 \\
		\hline
	\end{tabular}
\end{table}
As shown in Table \ref{tab:ecr_states}, at the diffidence function level, $T_{\min}$, $T_{\mathrm{prod}}$, and $S_{\mathrm{prob}}$ correspond to cautious, conjunctive, and disjunctive attitudes, respectively. The proposed PECR preserves this semantic order in the belief function domain. As shown in Table~\ref{tab:conflict_compare_caucr}, the conflict degree decreases from $0.136\rightarrow0.105\rightarrow0$, while the ignorance degree increases from $1.268\rightarrow1.344\rightarrow2.586$. This trend is consistent with the expected transition from more conjunctive to more disjunctive fusion semantics.

In contrast, CCR, CauCR, and DCR do not exhibit such a consistent semantic monotonicity. In particular, CauCR assigns almost all mass to the empty set, producing an extremely degenerate result, which does not conform to its \textit{cautious} semantics. The lower entropy of the minimum-based result should be interpreted at the decision-probability level after PPT normalization, rather than as a contradiction to the cautious semantics of the minimum operator at the representation level. In summary, the proposed fusion framework is more consistent with the corresponding triangular norm semantics.

\subsubsection{Flexible combination rules via triangular norm family}

As discussed in Section \ref{intro}, the second motivation of this paper is to provide a pathway for combining evidential information on possibilistic structure to realize the flexible combination rules via triangular norm family. Existing rules usually bind the source state and the whole mass structure to a single combination semantics. In the proposed framework, the singleton-level propensity and the high-order commitment are represented by different components and can be fused by different operators. This provides a flexible way to handle heterogeneous information, especially probability-possibility fusion.

Consider the Bayesian BPA
\[
m_{\rm prob}=\{0,0.2,0.5,0,0.3,0,0,0\},
\]
and the consonant mass function
\[
m_{\rm poss}=\{0,0,0,0,0.3,0,0.3,0.4\}.
\]
Their isopignistic relative functions are
\[
\begin{aligned}
	\widehat I_{m_{\rm prob}}
	&=
	\{0,0.6,1.0,0,0.8,0,0,0\},\\
	\widehat I_{m_{\rm poss}}
	&=
	\{0,0.4,0.7,0.308,1,0.308,1,1\}.
\end{aligned}
\]
The Bayesian BPA contains only singleton-level information, whereas the consonant mass function carries explicit higher-order structural commitment. Therefore, this example naturally reflects a heterogeneous fusion scenario: one source provides probabilistic decision tendency, while the other provides possibilistic structural uncertainty. We fix the propensity operator as the product t-norm and vary only the commitment operator. The results are shown in Table~\ref{tab:configuration_results}, where avg. indicates the average operator.

\begin{table}[!t]
	\centering
	\caption{Fusion results with $P=\mathrm{product}$ under different $C$ settings.}
	\label{tab:configuration_results}
	\small
	\setlength{\tabcolsep}{1.5pt}
	\renewcommand{\arraystretch}{1.15}
	\begin{tabular}{c|c c}
		\hline
		$\mathrm{C}$ & Fused BPA $m$ & $\mathrm{Ign}(m)$ \\
		\hline
		min
		&
		$\{0.2,0.08,0.31,0,$
		$0.41,0,0,0\}$
		& 0.8 \\
		
		avg.
		&
		$\{0.2,0.04,0.225,0.02,$
		$0.325,0.02,0.11,0.06\}$
		& 1.07 \\
		
		max
		&
		$\{0.2,0,0.14,0,$
		$0.24,0,0.18,0.24\}$
		& 1.46 \\
		\hline
	\end{tabular}
\end{table}

As shown in Table~\ref{tab:configuration_results}, three configurations produce the same empty set mass, $m(\emptyset)=0.2$. Fused results show that the proposed framework does not force probability-possibility fusion to over-depend on either side. If a conservative fusion of structural commitment is desired, the commitment operator can suppress unsupported higher-order information; if the possibilistic source is considered informative, a more permissive operator can preserve its structural uncertainty. Meanwhile, the propensity-level fusion remains fixed and continues to control the singleton decision tendency. Therefore, the proposed method allows the decision tendency and the epistemic uncertainty to be adjusted independently, providing a more diverse and flexible fusion process for heterogeneous probabilistic and possibilistic evidence.

\section{Parametric Possibilistic Evidence Combination Rule}
\label{sec:parametric-pecr}

Existing combination rules are usually designed for specific source states. In practical applications, however, the relationship among sources is often continuous and difficult to accurately characterize by a few discrete fusion modes, making it important to develop parameterized combination rules that can adaptively represent source relationships. The PECR can be generalized by replacing fixed triangular operators with parametric ones.

\subsection{Generalization through parametric triangular norms}

When the fixed fusion operators in Definition~\ref{def:pecr} are replaced by parametric triangular operators, the corresponding PECR can be regarded as a parametric PECR. Specifically, the propensity operator and the commitment operator can be selected from parametric t-norm or t-conorm families, denoted by $\circledsmall{\rm P}_{T_{\lambda_1}}$ and $\circledsmall{\rm C}_{S_{\lambda_2}}$, respectively. The resulting fused BPA is denoted as
\[
m^\circ_{T_{\lambda_1},S_{\lambda_2}}
=
\mathcal R
\left(
\widehat I^\circ_{T_{\lambda_1},S_{\lambda_2}}
\right),
\]
where $T_{\lambda_1}$ and $S_{\lambda_2}$ represent triangular operators parameterized by $\lambda$.

\begin{example}
	Consider $2$ BPAs:
	\begin{align}
		m_1=&\{0,0.18,0.12,0.05,0.22,0.08,0.15,0.20\},\\
		m_2=&\{0,0.07,0.20,0.10,0.12,0.18,0.08,0.25\}.
	\end{align}
	Using the Frank t-norm operator, we fix $\lambda_c=0.5$ and vary $\lambda_p\in(0,1)$, and then fix $\lambda_p=0.5$ and vary $\lambda_c\in(0,1)$, where $\lambda_p$ and $\lambda_c$ are parameters of propensity and commitment operators. Fig.~\ref{fig:frank_sensitivity} shows the resulting ranges of fused masses.
\end{example}

\begin{figure}[t]
	\centering
	\includegraphics[width=\columnwidth]{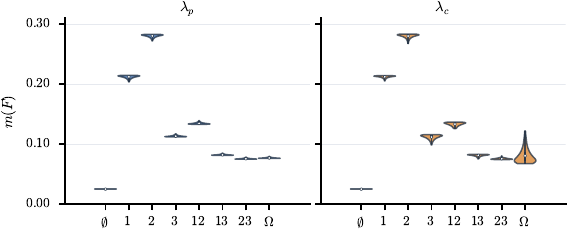}
	\caption{Variations of fused masses when the parameters change.}
	\label{fig:frank_sensitivity}
\end{figure}
As shown in Fig.~\ref{fig:frank_sensitivity}, changing the parameters of the triangular operators leads to observable variations in the fused masses. The parameter $\lambda_p$ mainly affects the singleton-level propensity fusion, while $\lambda_c$ controls the preservation and redistribution of higher-order commitment. This demonstrates that the parametric PECR forms a tunable family of fusion rules, enabling the fusion behavior to adapt continuously to different source relationships rather than being restricted to several fixed semantic modes.

\subsection{Applications for multi-view classification}

Multi-view classification \cite{huang2025esurvfusion} provides a natural test bed for demonstrating the advantage of parametric PECR. Fixed rules implicitly impose fixed assumptions on source dependence. The minimum operator reflects a cautious or redundant-view attitude, the product operator corresponds to conjunctive fusion under a more independent-source assumption, and majority voting ignores the strength of evidential support by only counting preferred classes. However, real views are usually neither fully redundant nor fully independent. To address this issue, we apply parametric PECR to multi-view classification, as shown in Fig.~\ref{application}. By tuning triangular-operator parameters, PECR can continuously adjust its fusion behavior between different fixed-rule semantics, thereby providing a more flexible mechanism for combining heterogeneous view information.

\begin{figure}[htbp!]
	\centering
	\includegraphics[width=0.49\textwidth]{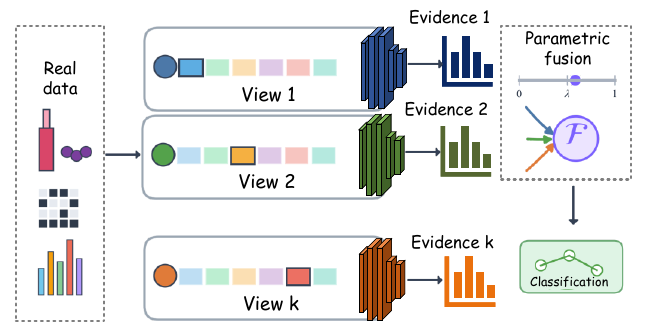}
	\caption{Overview of multi-view classification.}
	\label{application}
\end{figure}

\begin{table}[t]
	\caption{Multi-view protocols used in real-data experiments.}
	\label{tab:view_protocols}
	\centering
	\scriptsize
	\setlength{\tabcolsep}{2.5pt}
	\renewcommand{\arraystretch}{1.12}
	\begin{tabularx}{\columnwidth}{@{}l l l X@{}}
		\toprule
		Protocol & Data & Partition & Overlap setting \\
		\midrule
		Wine-C1 
		& Wine 
		& Contiguous 
		& Each view borrows one feature from each other view; sizes $7/6/6$. \\
		
		D0--4-R6 
		& Digits 0--4 
		& Round-robin 
		& Each view borrows six pixels from each other view; sizes $34/33/33$. \\
		
		D0--4-D4 
		& Digits 0--4 
		& Diagonal 
		& Each view borrows four pixels from each other view; sizes $29/30/29$. \\
		
		D5--9-R2 
		& Digits 5--9 
		& Round-robin 
		& Each view borrows two pixels from each other view; sizes $26/25/25$. \\
		
		BC-R4 
		& BreastCancer 
		& Round-robin 
		& Each view borrows four features from each other view; sizes $18/18/18$. \\
		\bottomrule
	\end{tabularx}
\end{table}

Consider a multi-view classification problem with class set
$\Omega=\{\omega_1,\omega_2,\ldots,\omega_c\}$. Each sample $x$ is represented by $k$ feature views, $x=(x^{(1)},x^{(2)},\ldots,x^{(k)})$,
where $x^{(i)}$ denotes the feature vector of the $i$-th view. For each view, a base classifier is trained independently and outputs a class-evidence vector
\[
e_i(\omega_j\mid x),\quad j=1,\ldots,c.
\]
The probabilistic output of each classifier is converted into a Bayesian BPA by $m_i(\{\omega_j\})=e_i(\omega_j\mid x)$ and $m_i(F)=0$ for $|F|\ne1$. Gaussian Naive Bayes is used as the base classifier for each view. The evaluation follows repeated stratified $5$-fold cross-validation with $5$ repeats. In each outer fold, preprocessing is performed without test leakage: the scaler is fitted only on the training split and then applied to both the training and test splits. All fusion rules are evaluated under the same train/test partitions. Table~\ref{tab:view_protocols} summarizes the $5$ multi-view protocols. Wine and BreastCancer are attribute-based datasets, while the Digits protocols are pixel-based classification tasks.

\begin{table*}[bp!]
	\caption{Real-data fusion results. Accuracy is reported as mean $\pm$ standard deviation. The best result in each row is bold.}
	\label{tab:real_fusion_results}
	\centering
	\footnotesize
	\setlength{\tabcolsep}{3.4pt}
	\renewcommand{\arraystretch}{1.18}
	\begin{tabular*}{\textwidth}{@{\extracolsep{\fill}}lccccccc@{}}
		\toprule
		Protocol & Frank & Hamacher & Min t-norm & Prod t-norm & CCR & CauCR & Majority\\
		\midrule
		Wine-C1 &
		\textbf{0.9786 $\pm$ 0.0219} &
		\textbf{0.9786 $\pm$ 0.0219} &
		0.9730 $\pm$ 0.0222 &
		0.9775 $\pm$ 0.0230 &
		0.9763 $\pm$ 0.0225 &
		0.9011 $\pm$ 0.0439 &
		0.9562 $\pm$ 0.0324\\
		D0--4-R6 &
		0.9099 $\pm$ 0.0349 &
		0.9101 $\pm$ 0.0342 &
		0.9052 $\pm$ 0.0376 &
		0.9050 $\pm$ 0.0322 &
		0.9039 $\pm$ 0.0322 &
		0.7889 $\pm$ 0.0534 &
		\textbf{0.9159 $\pm$ 0.0281}\\
		D0--4-D4 &
		0.8602 $\pm$ 0.0430 &
		\textbf{0.8611 $\pm$ 0.0433} &
		0.8599 $\pm$ 0.0431 &
		0.8393 $\pm$ 0.0452 &
		0.8311 $\pm$ 0.0450 &
		0.7239 $\pm$ 0.0528 &
		0.8222 $\pm$ 0.0383\\
		D5--9-R2 &
		0.8993 $\pm$ 0.0169 &
		\textbf{0.8998 $\pm$ 0.0170} &
		0.8991 $\pm$ 0.0168 &
		0.8817 $\pm$ 0.0230 &
		0.8794 $\pm$ 0.0261 &
		0.7190 $\pm$ 0.0378 &
		0.8283 $\pm$ 0.0258\\
		BC-R4 &
		\textbf{0.9438 $\pm$ 0.0263} &
		\textbf{0.9438 $\pm$ 0.0263} &
		0.9409 $\pm$ 0.0250 &
		0.9420 $\pm$ 0.0271 &
		0.9416 $\pm$ 0.0280 &
		0.9353 $\pm$ 0.0292 &
		0.9430 $\pm$ 0.0284\\
		\midrule
		Average &
		0.9183 $\pm$ 0.0450 &
		\textbf{0.9187 $\pm$ 0.0446} &
		0.9156 $\pm$ 0.0430 &
		0.9091 $\pm$ 0.0534 &
		0.9065 $\pm$ 0.0560 &
		0.8136 $\pm$ 0.1001 &
		0.8931 $\pm$ 0.0637\\
		\bottomrule
	\end{tabular*}
\end{table*}

Table~\ref{tab:real_fusion_results} shows that the parametric rules are more flexible and effective in the multi-view setting. Although no single fusion rule dominates all protocols, the two parametric rules achieve the best average performance. This supports the motivation that a single fixed rule is insufficient to describe all view relations, whereas a parametric operator family can adapt to different dependency patterns among views.

The selected candidates also provide useful semantic information. Product-like choices indicate that some views behave closer to independent evidence sources, while min-like or small-parameter choices suggest stronger overlap or redundancy. Thus, parametric PECR is not merely a numerical tuning strategy; it offers a data-driven way to adjust the fusion attitude and better handle heterogeneous multi-view evidence.

\section{Conclusion}
\label{con}

This paper proposed a possibilistic evidence combination framework for flexible fusion of belief functions. To address the difficulty of applying triangular operators directly to the belief function domain, we first constructed an isopignistic representation process that maps a BPA into an isopignistic relative function. Based on this representation, the singleton layer characterizes atomic propensity and the higher-order layers characterize structural commitment. The proposed PECR performs fusion in this possibilistic representation space: the propensity component and the commitment component can be fused by different triangular operators, and the fused representation can be reconstructed into a valid BPA. This design extends cautious, conjunctive, and disjunctive fusion semantics to the entire belief function domain, supports non-distinct sources through idempotent operators, and enables more flexible probability-possibility fusion. Numerical examples and real-data multi-view classification experiments show that the proposed framework can preserve structural uncertainty, reduce inappropriate conflict accumulation, and adapt fusion behavior through parametric triangular operators.

Future work will focus on both theoretical and application-oriented extensions. From the theoretical perspective, more general operator families and source-relation measures can be incorporated to automatically select suitable propensity and commitment fusion attitudes. The relationship between the proposed layer-wise possibilistic representation and other uncertainty measures, such as entropy, specificity, and dependence indices, also deserves further investigation. From the application perspective, the proposed framework can be extended to more complex multi-source scenarios, including sensor fusion, multimodal recognition, and decision systems with heterogeneous classifiers or large-scale foundation models.

\section*{Acknowledgments}
Authors are grateful to the reviewers for their valuable comments. This work is partially supported by the National Natural Science Foundation of China (Grant No. 62373078), and by the China Scholarship Council (202206070008).




\ifCLASSOPTIONcaptionsoff
  \newpage
\fi
\footnotesize
\bibliography{mybibfile}
\bibliographystyle{IEEEtran}

%









\end{document}